\documentclass[12pt,a4paper]{article}

\usepackage[utf8]{inputenc}
\usepackage[T1]{fontenc}
\usepackage[protrusion=true,expansion=false]{microtype}
\usepackage{geometry}
\usepackage{amsmath,amssymb,amsthm}
\usepackage{mathtools}

\usepackage{graphicx}
\usepackage{booktabs}

\usepackage{algorithm}
\usepackage{algpseudocode}

\usepackage{setspace}
\usepackage{parskip}
\usepackage{natbib}

\usepackage[colorlinks=true, linkcolor=blue, citecolor=blue, urlcolor=blue]{hyperref}
\usepackage{cleveref}
\crefname{proposition}{Proposition}{Propositions}

\usepackage{titlesec}
\titleformat{\section}{\large\bfseries}{\thesection.}{0.7em}{}
\titleformat{\subsection}{\normalsize\bfseries}{\thesubsection.}{0.7em}{}

\newtheorem{theorem}{Theorem}[section]

\newtheorem{proposition}[theorem]{Proposition}

\newtheorem{definition}[theorem]{Definition}
\newtheorem{remark}[theorem]{Remark}

\newcommand{\R}{\mathbb{R}}
\newcommand{\norm}[1]{\left\lVert#1\right\rVert}
\newcommand{\Ogap}{\Omega}
\newcommand{\supp}{\operatorname{supp}}
\newcommand{\VQ}{V_{\mathrm{OQ}}}
\newcommand{\Sspace}{\mathcal{S}}

\title{\textbf{State Representation and Termination for Recursive Reasoning Systems}}

\author{%
  Debashis Guha\thanks{S P Jain School of Global Management;
    \texttt{debashis.guha@spjain.org}} \quad
  Amritendu Mukherjee\thanks{Indian Statistical Institute;
    \texttt{amritendum@alum.iisc.ac.in}} \\[4pt]
  Sanjay Kukreja\thanks{S P Jain School of Global Management;
    \texttt{sanjay.ds18dba008@spjain.org}} \quad
  Tarun Kumar\thanks{eClerx Services Ltd.;
    \texttt{Tarun.Kumar06@eclerx.com}}
}
\date{}

\begin{document}
\maketitle

\begin{abstract}
Recursive reasoning systems alternate between acquiring new evidence and
refining an accumulated understanding. Two design choices are typically
left implicit: how to represent the evolving reasoning state, and when
to stop iterating. This paper addresses both.\\
We represent the reasoning state as an \emph{epistemic state graph}
encoding extracted claims, evidential relations, open questions, and
confidence weights. We define the \emph{order-gap} as the
distance between the states reached by expand-then-consolidate versus
consolidate-then-expand; a small order-gap suggests that the two
orderings agree and further iteration is unlikely to help. Our main result 
gives a necessary and sufficient condition for the
linearised order-gap to be non-degenerate near the fixed point, showing
when the criterion is informative rather than algebraically vacuous. This
is a local condition, not a global convergence guarantee. We apply the
framework to recursive reasoning systems and sketch its
application to agent loops, tree-of-thought reasoning, theorem proving,
and continual learning.
\end{abstract}

\newpage

\section{Introduction}
\label{sec:intro}

\subsection{Recursive Reasoning as Adaptive Inference}

A growing class of machine learning systems operates by iterative
refinement rather than single-pass inference. At each step such a system
acquires new information (a retrieved document, an observed action
outcome, a sampled hypothesis, or a newly generated chain-of-thought step)
and integrates it into an accumulated internal state before deciding
whether to iterate further or commit to an output. This pattern appears
in iterative retrieval-augmented generation \citep{trivedi2023ircot},
agent action-observation loops \citep{yao2023react}, recursive language
models \citep{zhang2025rlm}, tree-of-thought and graph-of-thought
reasoning \citep{yao2023tot,besta2023got}, and continual learning
\citep{kirkpatrick2017ewc}. We call this class \emph{recursive reasoning
systems}.

Despite their diversity, recursive reasoning systems share a common
structure. There is an internal state $s_t$ representing
what the system holds at step $t$. There is an \emph{expansion} step that
brings in new external evidence. There is a \emph{consolidation} step that
refines the state using only what is already present. And there is a
termination decision: iterate further, or commit.

\subsection{Two Missing Components}

In current systems, both the state and the termination criterion are handled
implicitly or imposed from outside, and their absence leads to characteristic
failure modes across four well-studied architectures as discussed below

\textbf{Chain-of-thought and tree-of-thought reasoning.}
Chain-of-thought prompting \citep{wei2022cot} and its tree-structured
extension \citep{yao2023tot} represent the reasoning state as a text
trace. There is no explicit record of which claims are supported, which
are in conflict, and which remain unresolved. Contradictions accumulate
without detection; reasoning drifts across steps; and there is no
mechanism to revisit or revise earlier conclusions in light of later
evidence. Termination is fixed by depth or iteration count, with no
notion of reasoning being complete.

\textbf{Retrieval-augmented generation and multi-hop QA.}
Iterative retrieval systems such as IRCoT \citep{trivedi2023ircot}
alternate between retrieving evidence and generating reasoning steps.
Retrieved facts are not tracked relationally across iterations,
provenance is not preserved in a structured way, and the system has no
representation of what it still needs to find. This leads to redundant
retrieval, missing critical evidence, and inconsistent synthesis across
documents. Termination is a fixed hop count or token budget, with no
signal whether further retrieval would change the answer or not.

\textbf{Agent action-observation loops.}
In ReAct-style agents \citep{yao2023react}, the state is a sequence of
actions and observations with no canonical representation of goals,
subgoals, or resolved versus unresolved tasks. This produces looping
behaviour, repeated tool calls, and failure to recognise completion.
Termination is an arbitrary maximum step count or manual stop condition.

\textbf{Iterative self-refinement.}
Self-Refine \citep{madaan2023selfrefine} and related approaches iterate
a generate-critique-refine loop without an explicit representation of
what has been fixed, what remains uncertain, or what is stable. Without
such a state, the system has no signal for when further refinement is
beneficial. Empirically, intrinsic self-correction without external
feedback frequently fails to improve and can degrade performance
\citep{huang2023selfcorrect}: the system oscillates between answers,
over-refines, or hallucinates improvements. Termination is a fixed
iteration count or a subjective heuristic.

Across all four cases, state is implicit in transient text or hidden
activations rather than structured, persistent, and inspectable; and
termination is imposed by fixed limits rather than derived from the
system's own trajectory.

\subsection{The Proposal}

We propose treating state representation and termination as explicit
design choices, and give concrete definitions to both.\\
For state, we introduce the \emph{epistemic state graph}
(\Cref{sec:state}): a structured graph whose nodes encode claims,
partial answers, and open questions, and whose edges encode evidential
support, logical dependency, and inconsistency, all weighted by
confidence. This representation is structured, persistent, and
inspectable across iterations.\\
For termination, we introduce the \emph{order-gap} (\Cref{sec:termination}):
a criterion derived from the system's own operator structure that measures
whether the system's state is still sensitive to the ordering of expansion
and consolidation. When the order-gap is small, both orderings produce
nearly the same result, suggesting that further iteration is unlikely to
matter. When it is large, the ordering still matters and the system has
not yet settled.

\subsection{Contributions}

\begin{enumerate}
  \item \textbf{State representation} (\Cref{sec:state}). The
    epistemic state graph, with a smooth Euclidean embedding
    that makes the expansion and consolidation operators amenable to
    analysis.
  \item \textbf{Operator framework} (\Cref{sec:operators}). Formal
    definitions of the expansion operator $P_e$ and consolidation operator
    $Q$ on the epistemic state, with their dynamics and key properties.
  \item \textbf{Order-gap termination} (\Cref{sec:termination}). The
    order-gap termination criterion, a windowed stopping rule, and a
    non-degeneracy theorem (\Cref{prop:gramian}) characterising when the
    criterion is informative.
  \item \textbf{Stopping algorithm} (\Cref{sec:alg}). Pseudocode for
    recursive reasoning with order-gap termination.
  \item \textbf{Other recursive reasoning systems} (\Cref{sec:other}).
    How the framework applies to agent loops, tree-of-thought reasoning,
    theorem proving, and continual learning.
  \item \textbf{Illustration} (\Cref{sec:sim}). A closed-form 2-dimensional
    example verifying the dynamics and order-gap formula.
\end{enumerate}

\section{State Representation}
\label{sec:state}

\subsection{Requirements for a Reasoning State}

A recursive reasoning system iterates by processing new evidence and
refining its accumulated understanding. The state must support two
operations: expansion, which adds information from a new piece of
evidence; and consolidation, which refines what is already present.
For both operations to be principled, the state must make their inputs
and outputs explicit. Expansion needs to know what is already established
(to avoid redundancy) and what is still open (to direct retrieval).
Consolidation needs to know which claims conflict (to resolve them) and
which are mutually supporting (to strengthen them). The following six
objects are the minimum needed to support both operations correctly:

\begin{itemize}
  \item \textbf{Claims}: extracted facts, with their provenance and
    confidence.
  \item \textbf{Evidential relations}: which claims support which others,
    and how strongly.
  \item \textbf{Conflicts}: which claims are inconsistent with each other.
  \item \textbf{Partial conclusions}: the system's current best answer to
    the question or sub-questions, together with the claims that support it.
  \item \textbf{Open questions}: dependencies that have been identified but
    not yet resolved: things the system knows it needs but has not yet found.
  \item \textbf{Confidence}: a weight on each claim and each relation,
    reflecting how strongly the evidence supports it.
\end{itemize}

Without these, the system cannot direct later extraction based on what
earlier extraction found, cannot detect when new evidence contradicts what
it has already consolidated, and cannot identify what it still needs. A
state that holds all of this explicitly is the foundation for principled
expansion and consolidation.

\subsection{The Epistemic State Graph}

\begin{definition}[Epistemic State Graph]
\label{def:esg}
An \emph{epistemic state graph} is a tuple $S = (V, E, \ell, c, w)$ where:
\begin{itemize}
  \item $V = V_C \cup V_A \cup \VQ$ is a finite vertex set partitioned
    into \textbf{Claim} nodes ($V_C$), \textbf{PartialAnswer} nodes
    ($V_A$), and \textbf{OpenQuestion} nodes ($\VQ$).
  \item $E \subseteq V \times V \times \mathcal{T}$ is a set of typed
    directed edges, with $\mathcal{T} = \{\textsc{Supports},
    \textsc{Requires}, \textsc{Contradicts}\}$.
  \item $\ell : V \to \R^k$ assigns a $k$-dimensional attribute vector
    to each node, encoding the claim text embedding and source identifier.
  \item $c : V \to (0,1]$ assigns a confidence weight to each node.
  \item $w : E \to (0,1]$ assigns a confidence weight to each edge.
\end{itemize}
\end{definition}

\textbf{Node types.} Claim nodes record extracted facts. PartialAnswer
nodes hold the system's current best answer to the question or a
sub-question, together with its supporting claims. OpenQuestion nodes
mark dependencies the system has identified but not yet resolved.

\textbf{Edge types.} \textsc{Supports} edges run from claims to claims or
to partial answers, encoding evidential backing. \textsc{Requires} edges
run from partial answers or claims to open questions, encoding logical
dependency on unresolved evidence. \textsc{Contradicts} edges connect
claims that are mutually inconsistent given current evidence.

\textbf{Consistency.} The graph is \emph{consistent} at threshold
$\delta \in (0,1)$ if no two nodes $u, v$ connected by a
\textsc{Contradicts} edge satisfying $c(u) > \delta$ and $c(v) >
\delta$ at the same time. Consolidation drives the graph toward consistency.

\subsection{Smooth Euclidean Embedding}
\label{sec:embed}

The epistemic graph is a combinatorial object; raw graph operations such
as adding nodes, merging nodes, and retyping edges are discrete and not
differentiable. The operator analysis in \Cref{sec:termination} requires
smooth maps. We therefore embed the graph in a fixed-dimension Euclidean
space using differentiable relaxations.

\begin{definition}[Graph Embedding]
\label{def:embed}
Fix a maximum node count $N_{\max}$ and attribute dimension $k$. The
\emph{graph embedding} $\varphi : \mathcal{G} \to \R^d$ encodes the
graph state by:
\begin{enumerate}
  \item Encoding each node by its attribute vector $\ell(v) \in \R^k$
    together with its confidence weight $c(v) \in (0,1]$, padded to
    $N_{\max}$ slots in a fixed canonical order by node type, giving a
    block of dimension $(k+1)N_{\max}$.
  \item Appending a flattened adjacency-weight tensor indexed by ordered
    node pairs and edge type, of dimension $N_{\max}^2 |\mathcal{T}|$.
\end{enumerate}
The total dimension is $d = (k+1)N_{\max} + N_{\max}^2 |\mathcal{T}|$.
\end{definition}

The embedding $\varphi$ converts $S$ into a vector $\theta = \varphi(S)
\in \R^d$, with graph operations implemented as differentiable coordinate
updates so that $P_e$ and $Q$ are $C^1$ on a neighbourhood of the states
of interest. We work in $\R^d$ with the Euclidean norm throughout.

\section{Expansion and Consolidation Operators}
\label{sec:operators}

\subsection{State Space and Operators}

Let $(\Sspace, \norm{\cdot})$ be a normed complete (Banach) space. The
embedded state $\theta_t = \varphi(S_t) \in \R^d \subset \Sspace$
represents the system's accumulated understanding at step $t$.

The \emph{expansion operator} $P_e : \Sspace \to \Sspace$ updates the
state by incorporating a new piece of evidence $e \in \mathcal{E}$, drawn
from a distribution $P(\cdot \mid \theta_t)$ that may depend on the
current state. In the epistemic graph setting, $P_e$ adds Claim nodes for
newly extracted facts, links them to existing nodes via typed edges, and
promotes any OpenQuestion node that $e$ directly resolves to a
PartialAnswer node. Expansion is the only means by which new external
information enters the state.

The \emph{consolidation operator} $Q : \Sspace \to \Sspace$ refines the
state using only what is already present, without acquiring new evidence.
In the epistemic graph, $Q$ resolves \textsc{Contradicts} pairs by
removing or downweighting the lower-confidence endpoint, merges redundant
Claim nodes, and shifts weight toward the current best PartialAnswer.

The expansion-consolidation decomposition and its formal properties are
developed in full generality in \citet{guha2026opmech}; this paper
applies the decomposition to recursive reasoning systems and uses
it to derive the order-gap termination criterion.

For the operator analysis, we impose the following assumption. The
concrete graph operations described above do not automatically satisfy it
under all embeddings; it is a condition on the implemented consolidation
map, which must be verified for a given implementation.

\begin{description}
  \item[Assumption (Contractivity).] $Q$ is a \emph{contraction} on
    $(\Sspace, \norm{\cdot})$: there exists $\rho \in [0,1)$ such that
    $\norm{Q(\theta_1) - Q(\theta_2)} \leq \rho\norm{\theta_1 - \theta_2}$
    for all $\theta_1, \theta_2 \in \Sspace$.
\end{description}

Since $\Sspace$ is a normed complete space and $Q$ maps $\Sspace$ into
itself, the Banach fixed-point theorem then guarantees a unique fixed
point $\theta^\star \in \Sspace$ satisfying $Q(\theta^\star) =
\theta^\star$. This is a fixed point of the consolidation map $Q$ alone,
not necessarily of the full stochastic dynamics $\theta_{t+1} =
Q(P_{e_t}(\theta_t))$, which depends on the expansion operator and the
evidence distribution as well.

\subsection{Dynamics}

The update rule at each recursive reasoning step is:
\begin{equation}
\label{eq:dynamics}
  \theta_{t+1} = Q(P_{e_t}(\theta_t)), \qquad e_t \sim P(\cdot \mid
  \theta_t).
\end{equation}
The system expands to incorporate new evidence, then consolidates what
it holds. This is the structure underlying the Recursive Language Model's
read-then-aggregate loop \citep{zhang2025rlm} and the Tiny Recursive
Model's $z$-then-$y$ iteration \citep{jolicoeurmartineau2025trm}; the
two systems arrive at the same expand-then-consolidate pattern from
opposite ends of the model-scale spectrum.

\subsection{Non-Commutativity}

A key structural property of the operators $P_e$ and $Q$ is that they
do not, in general, commute. Consolidating before expansion commits the
state to its current best answer before new evidence arrives; expanding
before consolidation allows new evidence to influence the state before
commitment. These two orderings produce different results, and the
magnitude of the difference reflects how much the system's current
understanding would change if it paused to look at one more piece of
evidence before consolidating. This difference is the quantity we use as
a termination criterion.

\section{The Termination Problem}
\label{sec:term_problem}

\subsection{Why Termination is Hard Without State}

Without an explicit state, the termination decision has nothing to act on.
A system whose state is an unstructured text trace cannot inspect whether
it has found evidence for all its open questions, whether conflicts have
been resolved, or whether further evidence would change the answer. It
must rely on external proxies: a step count, a token budget, or a
self-reported confidence level.

Each proxy fails predictably. Step counts are right for neither easy
instances (which converge early) nor hard ones (which may need more depth).
Token budgets conflate the cost of computation with the value of additional
reasoning. Self-reported confidence is systematically miscalibrated on
precisely the instances that are hardest \citep{kadavath2022calibration}.

\subsection{What a Principled Termination Criterion Needs}

A principled termination criterion should have three properties. First,
it should be \emph{endogenous}: derived from the system's own state
trajectory, not imposed from outside. Second, it should be \emph{sensitive}:
it should be large when further iteration would materially change the
answer and small when the system has settled. Third, it should be
\emph{computable}: calculable from quantities the system already has
access to, without requiring knowledge of the ground truth or the full
document.

The termination criterion we propose in \Cref{sec:termination} is
designed to satisfy all three. It is derived from the operator
structure of the system itself, it reflects whether the ordering of
operations still matters, and it is computable at each step from the
current state and the current evidence.

\section{Order-Gap Termination}
\label{sec:termination}

\subsection{The Order-Gap Criterion}

\begin{definition}[Order-Gap {\normalfont\citep{guha2026opmech}}]
\label{def:ogap}
The \emph{order-gap} at state $\theta \in \Sspace$ given evidence $e$ is
\begin{equation}
\label{eq:ogap}
  \Ogap(\theta;\, e) \;=\; \norm{Q(P_e(\theta)) - P_e(Q(\theta))}.
\end{equation}
\end{definition}

The definition holds the
realised evidence item $e$ fixed across both orderings. The first
ordering $Q(P_e(\theta))$ is what the system actually computes
(expand then consolidate). The second $P_e(Q(\theta))$ is the
hypothetical in which consolidation happened first. Since the
evidence distribution may depend on the state, $P_e(Q(\theta))$ uses
the same realized $e$ rather than resampling from $P(\cdot \mid
Q(\theta))$; this makes the comparison well-defined and computable
from quantities already available at each step.

When $\Ogap(\theta; e)$ is large, the two orderings of expansion and
consolidation produce substantially different states: the system's answer
is sensitive to whether it consolidates before or after reading the next
piece of evidence, indicating it has not settled. When $\Ogap(\theta; e)$
is small, both orderings produce nearly the same result: the ordering no
longer matters, and further expansion is less likely to change the outcome materially.

\subsection{Windowed Stopping Rule}

This motivates the following termination criterion: halt iteration when
the windowed empirical order-gap
\begin{equation}
\label{eq:stopping}
  \widehat{\Ogap}_{t,w} \;=\; \frac{1}{w}
  \sum_{\tau=t-w+1}^{t} \Ogap(\theta_\tau;\, e_\tau)
\end{equation}
falls below a threshold $\varepsilon > 0$.

\textbf{Scope.} A small order-gap is a useful termination criterion, but it
does not by itself guarantee that the state has converged to the true
answer. \Cref{prop:gramian} below establishes when the linearised
order-gap near the fixed point is non-degenerate, so that a small
order-gap is not an algebraic artefact. Connecting this to a global
convergence guarantee
(via Lipschitz constants, noise bounds, and finite-sample concentration
of $\widehat{\Ogap}_{t,w}$) would require additional assumptions and
is beyond the present scope.

\subsection{Linearised Commutator and Gramian}
\label{sec:linearise}

To characterise non-degeneracy formally, we linearise the operators at
the consolidation fixed point $\theta^\star$. We impose one additional
assumption:

\begin{description}
  \item[Assumption (Redundancy at the fixed point).]
    $P_e(\theta^\star) = \theta^\star$ for all $e \in \mathcal{E}$.
\end{description}

This says that at the fully consolidated fixed point, processing
additional evidence leaves the state unchanged, since all relevant
information is already incorporated. Under this assumption, the chain
rule applied to $F_e(\theta) = Q(P_e(\theta)) - P_e(Q(\theta))$ at
$\theta^\star$ gives $DF_e(\theta^\star) = DQ(P_e(\theta^\star))
DP_e(\theta^\star) - DP_e(\theta^\star) DQ(\theta^\star) =
DQ(\theta^\star) DP_e(\theta^\star) - DP_e(\theta^\star) DQ(\theta^\star)$,
where the last equality uses $P_e(\theta^\star) = \theta^\star$.
We define the \emph{linearised commutator} for evidence $e$ as the following
quantity:
\begin{equation}
\label{eq:commutator}
  \Sigma_e \;:=\; DQ(\theta^\star)\,DP_e(\theta^\star)
    - DP_e(\theta^\star)\,DQ(\theta^\star) \;\in\; \R^{d \times d}.
\end{equation}
The \emph{commutator Gramian} is
\begin{equation}
\label{eq:gramian_def}
  G_{\theta^\star} \;:=\; \mathbb{E}_{e \sim P(\cdot \mid \theta^\star)}
  \bigl[\Sigma_e^\top \Sigma_e\bigr]
  \;\in\; \R^{d \times d},
\end{equation}
where the expectation is taken under the evidence distribution evaluated
at the fixed point $\theta^\star$. For any $v \in \R^d$, the quadratic
form satisfies:
\begin{equation}
\label{eq:quad}
  v^\top G_{\theta^\star} v
  \;=\; \mathbb{E}_{e \sim P(\cdot \mid \theta^\star)}
  \bigl[\norm{\Sigma_e v}^2\bigr].
\end{equation}
For finite or countable evidence distributions this expands as
$\sum_{e \in \supp(P)} P(e \mid \theta^\star)\,\norm{\Sigma_e v}^2$;
for general evidence spaces the sum is replaced by an integral and the
support condition by a $P(\cdot \mid \theta^\star)$-almost-sure condition.
The proposition below is stated for the finite/countable case; the
general case follows by replacing sums with integrals throughout.

\subsection{Non-Degeneracy of the Commutator Gramian}

\begin{proposition}[Local Non-Degeneracy of the Commutator Gramian]
\label{prop:gramian}
Let $W \subseteq \R^d$ be a subspace. For finite or countable evidence
distributions, the commutator Gramian satisfies $G_{\theta^\star} \succ 0$
on $W$ (i.e., $v^\top G_{\theta^\star} v > 0$ for all nonzero $v \in W$)
if and only if
\begin{equation}
\label{eq:kernel}
  \bigcap_{e \in \supp(P(\cdot \mid \theta^\star))} \ker(\Sigma_e)
  \;\cap\; W \;=\; \{0\}.
\end{equation}
For general evidence spaces, the equivalent condition is: for every
nonzero $v \in W$, $\Sigma_e v \neq 0$ on a set of positive
$P(\cdot \mid \theta^\star)$-measure.
\end{proposition}

\begin{proof}
For finite or countable distributions,
$v^\top G_{\theta^\star} v = \sum_{e \in \supp(P(\cdot \mid \theta^\star))}
P(e \mid \theta^\star)\,\norm{\Sigma_e v}^2$. Since
$P(e \mid \theta^\star) > 0$ for all $e \in \supp(P(\cdot \mid \theta^\star))$,
each summand is non-negative. Hence $v^\top G_{\theta^\star} v = 0$ if and
only if $\norm{\Sigma_e v}^2 = 0$ for every $e$ in the support, i.e.,
$v \in \ker(\Sigma_e)$ for every such $e$. Therefore
$v^\top G_{\theta^\star} v = 0 \iff v \in \bigcap_{e \in \supp(P(\cdot
\mid \theta^\star))} \ker(\Sigma_e)$. Restricting to $W$:
$G_{\theta^\star} \succ 0$ on $W$ holds iff no nonzero $v \in W$ has
$v^\top G_{\theta^\star} v = 0$, which is exactly condition
\eqref{eq:kernel}. For the general-distribution case,
$v^\top G_{\theta^\star} v = \int \norm{\Sigma_e v}^2\,
dP(e \mid \theta^\star)$. Since the integrand is non-negative,
the integral is zero if and only if $\Sigma_e v = 0$ holds
$P(\cdot \mid \theta^\star)$-almost surely. Hence
$G_{\theta^\star} \succ 0$ on $W$ holds if and only if every nonzero
$v \in W$ satisfies $\Sigma_e v \neq 0$ on a set of positive
$P(\cdot \mid \theta^\star)$-measure.
\end{proof}

\begin{remark}[Scope and interpretation]
\label{rem:scope}
\Cref{prop:gramian} is a local result at $\theta^\star$. It
characterises when every direction in $W$ is moved by at least one
evidence item in $\supp(P)$ under the first-order approximation to the
commutator. When condition \eqref{eq:kernel} holds, the linearised
order-gap near $\theta^\star$ is not algebraically zero in any direction
of $W$; a small observed order-gap is therefore not explained merely
by algebraic cancellation in the linearised commutator. The condition does not imply that the current state $\theta_t$
is close to $\theta^\star$, nor does it bound the error of the current
answer. In particular, it does not imply global convergence or optimal
stopping. The expansion-consolidation operator framework, the order-gap
termination criterion, and global convergence properties under Lipschitz
and contraction conditions are developed in \citet{guha2026opmech}; the present paper
applies that framework to the recursive reasoning setting and derives the
non-degeneracy condition for the graph-structured case.
\end{remark}

\subsection{Per-Question Coverage Check}

\begin{remark}[Necessary coverage check]
\label{rem:coverage}
The kernel intersection condition \eqref{eq:kernel} must hold for all
nonzero $v \in W$, which is generally not feasible to verify directly.
A necessary check operates question by question. For each open-question
node $q \in \VQ(S)$, let $v_q = \varphi(S^{(q)}) - \varphi(S)$ be the
embedding direction corresponding to resolving $q$, where $S^{(q)}$ is
the state with $q$ resolved. Check whether there exists $e \in \supp(P)$
such that $\Sigma_e v_q \neq 0$, equivalently whether the evidence
distribution assigns positive probability to at least one item that bears
on $q$ through a \textsc{Supports} or \textsc{Requires} path.

This check is necessary but not sufficient for \eqref{eq:kernel}: a
linear combination of the $v_q$ may lie in the kernel intersection even
when each basis direction does not. Any open-question node failing the
check identifies a definite gap in the evidence that must be repaired
before applying the non-degeneracy condition.
\end{remark}

\section{Algorithm}
\label{sec:alg}

\Cref{alg:stop} formalises the recursive reasoning procedure with
order-gap termination. Each iteration computes both $Q(P_{e_t}(\theta_t))$
(the actual state update) and $P_{e_t}(Q(\theta_t))$ (the alternative
ordering, computed only to evaluate the order-gap). This roughly doubles the
per-iteration consolidation cost, which is typically small relative to
the cost of expansion (document retrieval, model forward passes).

\begin{algorithm}[h]
\caption{Recursive Reasoning with Order-Gap Termination}
\label{alg:stop}
\begin{algorithmic}[1]
\Require Initial state $\theta_0$; operators $P_e$, $Q$; evidence
  distribution $P(\cdot \mid \theta)$; threshold $\varepsilon > 0$;
  window width $w$; budget $T_{\max}$
\Ensure Final state $\theta_T$; stopping step $T$
\State Initialise buffer $B \leftarrow ()$
\For{$t = 0, 1, \ldots, T_{\max} - 1$}
  \State Sample $e_t \sim P(\cdot \mid \theta_t)$
  \State Compute $\Ogap_t \leftarrow \norm{Q(P_{e_t}(\theta_t)) -
    P_{e_t}(Q(\theta_t))}$
  \State Update $\theta_{t+1} \leftarrow Q(P_{e_t}(\theta_t))$
  \State Append $\Ogap_t$ to $B$
  \If{$|B| \geq w$ \textbf{and}
    $\frac{1}{w}\sum_{i=|B|-w+1}^{|B|} B_i \leq \varepsilon$}
    \State \Return $(\theta_{t+1},\; t+1)$
  \EndIf
\EndFor
\State \Return $(\theta_{T_{\max}},\; T_{\max})$
\end{algorithmic}
\end{algorithm}

\section{Application to Recursive Language-Model Reasoning}
\label{sec:rlm}

We now apply the framework to recursive language-model reasoning,
the setting in which the state and termination problems are sharpest and
the engineering infrastructure most developed.

\subsection{The Long-Context Problem and Recursive Approaches}

Single-pass transformer models read all tokens simultaneously. As document
length grows, performance on tasks requiring integration of information
across the document degrades, not because the model lacks capacity in
principle, but because attention, spread across tens of thousands of
tokens, diffuses \citep{liu2024lostmiddle}. The Recursive Language Model
(RLM), developed by \citet{zhang2025rlm} addresses this by treating the document as an external environment to be explored programmatically; the model reads
fragments, extracts structured information, and spawns recursive calls for
sub-questions, folding their results into a growing aggregate state. On
the OOLONG long-context benchmark, the RLM maintained strong performance
at one million tokens where single-pass models degraded substantially.

The Tiny Recursive Model (TRM) of \citet{jolicoeurmartineau2025trm},
developed independently for abstract visual reasoning, alternates a
$z$-update (refining the latent state given fixed input) with a $y$-update
(committing to an output). At seven million parameters it achieved 45\%
on ARC-AGI-1, surpassing models orders of magnitude larger.

Both systems follow the dynamics of \eqref{eq:dynamics} but leave its
two components implicit. The RLM's aggregate state is an unstructured
REPL buffer; the TRM's is a latent vector. Neither derives its
termination criterion from the operator structure.

\subsection{Operator Identification}

In the language-model setting, the epistemic graph of \Cref{def:esg} is
the state $S_t = \theta_t$. The operators are:

\textbf{Expansion.} $P_{e_t}(S)$ incorporates document chunk $e_t$: it
adds Claim nodes for newly extracted facts, links them to existing nodes
via typed edges, and promotes any OpenQuestion node that $e_t$ resolves
to a PartialAnswer node.

\textbf{Consolidation.} $Q(S)$ refines the accumulated state: it resolves
\textsc{Contradicts} pairs by removing or downweighting the lower-confidence
endpoint, merges redundant Claim nodes, and shifts weight toward the
current best PartialAnswer.

\subsection{Control Applications}

Three control applications follow directly from the order-gap termination criterion.

\textbf{Stopping.} Halt chunk processing when
$\widehat{\Ogap}_{t,w} \leq \varepsilon$. Assuming that reprocessing a
chunk already fully represented leaves the state unchanged, this follows
\Cref{alg:stop} with formal motivation from \Cref{prop:gramian} when the
coverage check of \Cref{rem:coverage} holds.

\textbf{Consolidation scheduling.} Trigger $Q$ when accumulated
order-gap since the last consolidation exceeds a cost threshold, rather
than on a fixed schedule.

\textbf{Adaptive extraction.} Weight chunk selection toward evidence that
addresses OpenQuestion nodes failing the coverage check of
\Cref{rem:coverage}, directing retrieval toward what is genuinely missing.

\section{Other Recursive Reasoning Systems}
\label{sec:other}

The expansion-consolidation decomposition and the order-gap termination
criterion apply beyond recursive language models. We sketch five
additional cases to establish the generality of the framework.

\textbf{Agent action-observation loops.} In systems such as ReAct-style
agents \citep{yao2023react}, the state is the agent's current understanding of its goals,
sub-goals, and environmental observations. Expansion is taking an action
and observing the outcome; consolidation is updating the internal goal
structure in light of the observation. The order-gap measures whether the
agent's goal understanding would shift if it took one more action before
updating its goal model. The failure modes of repeated tool calls and
looping correspond precisely to a large order-gap that the system
cannot detect because it has no state to measure it against.

\textbf{Tree-of-thought and graph-of-thought reasoning.} These systems
\citep{yao2023tot,besta2023got} generate multiple reasoning paths and
aggregate them.
Expansion is generating a new reasoning branch; consolidation is
aggregating evidence across branches. The order-gap measures whether a
new branch would change the aggregated conclusion before the current
branches are consolidated. The absence of an explicit state means that
contradictions across branches accumulate without detection, in analogy with
the \textsc{Contradicts} edge problem in the epistemic graph.

\textbf{Iterative self-refinement.}
Self-refinement systems \citep{madaan2023selfrefine} iterate a
generate-critique-refine loop. Expansion is generating a new candidate
output or critique; consolidation is integrating the critique into a
revised answer. The order-gap measures whether the system's current
answer would change if it consolidated its existing critique before
generating a new one. Without an explicit state tracking what has been
fixed and what remains uncertain, the system cannot detect when
refinement has converged, which is precisely the condition under which
further iteration degrades rather than improves the answer
\citep{huang2023selfcorrect}.

\textbf{Theorem proving.} In neural theorem provers \citep{lample2022hypertree},
the state is the current proof context: established lemmas, open subgoals,
and failed attempts. Expansion is applying a tactic or lemma; consolidation
is simplifying the proof context and closing sub-goals. The order-gap
measures whether the proof context is sensitive to the order in which
tactics are applied, a useful indicator of whether the current
strategy is well-founded.

\textbf{Continual learning.} A continual learning system \citep{kirkpatrick2017ewc}
must expand to incorporate new task evidence while consolidating to protect
previously learned knowledge. The order-gap measures the conflict between
new task evidence and existing knowledge: large values indicate that the
system is in a regime of high task conflict where plasticity and retention
trade against each other, and adaptive regularisation is warranted.

In all these cases the two components (an explicit structured state
and the order-gap as a termination criterion) apply naturally within
the general framework of \Cref{sec:operators,sec:termination}.

\section{Closed-Form Illustration}
\label{sec:sim}

We trace the order-gap on a synthetic example with analytically specified
operators to verify that the dynamics \eqref{eq:dynamics} and the
order-gap formula \eqref{eq:ogap} are internally consistent and to
illustrate the criterion's qualitative behaviour. This is expository, not
empirical.

\subsection{Operators and Closed-Form Order-Gap}

We use a 2-dimensional state $\theta = (c, u) \in [0,1]^2$, where $c$ is
the system's confidence in its current answer and $u$ is its residual
uncertainty. The operators are:
\begin{align}
  P_e(c,\, u) &= \bigl(c + \alpha_e(1 - c),\;\; u\bigr),
  \label{eq:Pe} \\
  Q(c,\, u)   &= \bigl(c,\;\; \rho(1 - c)\,u\bigr),
  \label{eq:Q}
\end{align}
where $\alpha_e \in (0,1)$ is the relevance weight of evidence $e$ and
$\rho \in (0,1)$ is the consolidation rate. $P_e$ raises confidence
proportionally to relevance; $Q$ reduces uncertainty proportionally to
remaining headroom $(1-c)$, so consolidation has diminishing effect as
confidence approaches one.

Both operators extend to $C^\infty$ maps on an open neighbourhood of
$[0,1]^2$. For $Q$ alone, every point
$(c, 0)$ is a fixed point. Under repeated relevant expansion with
$\alpha_e > 0$, the coupled dynamics $\theta_{t+1} = Q(P_{e_t}(\theta_t))$
drive $c \to 1$ and $u \to 0$, so $(1, 0)$ is the limiting terminal
state of the full process. This example does not satisfy the global
contraction assumption on the full $(c, u)$ state, since $Q$ leaves $c$
unchanged, but it illustrates the order-gap calculation and decay
behaviour in a transparent setting.

\textbf{Closed-form order-gap.} Direct computation gives:
\begin{align}
  Q(P_e(c,u)) &= \bigl(c + \alpha_e(1-c),\;\;
    \rho\,(1-c)(1-\alpha_e)\,u\bigr), \label{eq:QPe}\\
  P_e(Q(c,u)) &= \bigl(c + \alpha_e(1-c),\;\;
    \rho\,(1-c)\,u\bigr). \label{eq:PeQ}
\end{align}
The two orderings agree on the first coordinate and differ on the second:
\begin{equation}
\label{eq:gap-formula}
  \Ogap\bigl((c,u);\, e\bigr) \;=\; \rho\,\alpha_e\,(1-c)\,u.
\end{equation}
The signal vanishes as $c \to 1$ (high answer confidence), as $u \to 0$
(uncertainty exhausted), or as $\alpha_e \to 0$ (irrelevant evidence).
It is large when confidence is low, uncertainty is high, and the evidence
is relevant, the regime where further expansion is most likely to matter.

\subsection{Trajectory}

We set $\rho = 0.9$, initial state $\theta_0 = (0.20,\; 0.80)$, and
present five pieces of evidence with relevance weights $\alpha_1 = 0.05$,
$\alpha_2 = 0.35$, $\alpha_3 = 0.40$, $\alpha_4 = 0.05$, $\alpha_5 =
0.05$. Evidence items $e_2$ and $e_3$ are the pertinent ones; $e_1$,
$e_4$, $e_5$ are tangential. States $\theta_{t+1} = Q(P_{e_t}(\theta_t))$
are computed from \eqref{eq:Pe}--\eqref{eq:Q}; order-gap values
$\Ogap_t = \Ogap(\theta_{t-1};\, e_t)$ are computed from
\eqref{eq:gap-formula}. Values are computed from the closed-form equations
and rounded to three decimals.

\begin{table}[h]
\centering
\caption{Order-gap trajectory. The signal rises when pertinent evidence
($e_2$, $e_3$) arrives and falls sharply as confidence rises and
uncertainty falls. Tangential evidence ($e_4$, $e_5$) produces a near-zero
signal. Expository, not empirical.}
\label{tab:sim}
\begin{tabular}{ccccc}
\toprule
Step $t$ & Evidence & $\theta_t = (c_t, u_t)$ & $\Ogap_t$ &
  $\widehat{\Ogap}_{t,2}$ \\
\midrule
0 & --    & $(0.200,\; 0.800)$ & --      & --      \\
1 & $e_1$ & $(0.240,\; 0.547)$ & $0.029$ & --      \\
2 & $e_2$ & $(0.506,\; 0.243)$ & $0.131$ & $0.080$ \\
3 & $e_3$ & $(0.704,\; 0.065)$ & $0.043$ & $0.087$ \\
4 & $e_4$ & $(0.718,\; 0.016)$ & $0.001$ & $0.022$ \\
5 & $e_5$ & $(0.732,\; 0.004)$ & $0.000$ & $0.001$ \\
\bottomrule
\end{tabular}
\end{table}

\textbf{Sample computation (step 1).} $\theta_0 = (0.20, 0.80)$,
$\alpha_1 = 0.05$: $P_{e_1}(\theta_0) = (0.240, 0.800)$;
$\theta_1 = Q(0.240, 0.800) = (0.240,\; 0.9 \times 0.760 \times 0.800)
= (0.240,\; 0.547)$;
$\Ogap_1 = 0.9 \times 0.05 \times 0.800 \times 0.800 = 0.029$.

\textbf{Stopping behaviour.} With threshold $\varepsilon = 0.025$,
\Cref{alg:stop} halts at step 4 ($\widehat{\Ogap}_{4,2} = 0.022 <
0.025$) at which point $c_4 = 0.718$. A fixed budget of $T_{\max} = 5$
continues through $e_5$, reaching $c_5 = 0.732$, a change of $0.014$
relative to the step-4 state. The coverage check of \Cref{rem:coverage}
correctly identifies the failure mode: if $e_2$ and $e_3$ had zero
sampling probability, the order-gap would stay near zero throughout, not because the state had converged, but because the pertinent evidence
was never reached.

\section{Conclusion}
\label{sec:conclusion}

Recursive reasoning systems require two components that are
consistently left implicit in current practice - (i) an explicit, structured
representation of the evolving reasoning state, and (ii) a termination
criterion derived from the system's own dynamics rather than imposed by
an external budget.\\
We have proposed concrete definitions for both. The epistemic state graph, with its typed nodes, edges and smooth Euclidean embedding, provides a state that is structured, persistent, and amenable to operator analysis.
The order-gap, measuring the non-commutativity of
expansion and consolidation, provides a termination criterion that is
derived from the system's own dynamics, sensitive to whether further
iteration matters, and computable at each step.\\
The non-degeneracy theorem characterises exactly when the linearised criterion is
non-degenerate, meaning that a small order-gap is not explained merely by
algebraic cancellation in the linearised commutator. We are explicit that this is a local condition.\\
The framework applies naturally to recursive language-model
reasoning, agent loops, tree-of-thought
reasoning, theorem proving, and continual learning. The closed-form illustration 
confirms the qualitative behaviour of the criterion; large when pertinent
evidence still changes the state, small once the system has settled.\\
The near-term extensions are: a global convergence guarantee connecting
the non-degeneracy condition to answer error via Lipschitz and
concentration arguments, empirical evaluation of \Cref{alg:stop} against
fixed-budget and confidence-threshold baselines on long-context QA
benchmarks and further applications to continual learning and theorem
proving.


\end{document}